\documentclass{article}
\usepackage{spconf,amsmath,graphicx}

\usepackage[english]{babel}

\usepackage{ifpdf}

\usepackage{cite} 
\usepackage{url}
\usepackage{hyperref}
\usepackage{color}
\usepackage{pgf, tikz, pgfplots}
\usetikzlibrary{shapes, arrows, automata, plotmarks}
\usetikzlibrary{calc,hobby,decorations}
\usepackage{tikz-cd}

\usepackage{stfloats}

\usepackage{amsfonts, amssymb, amsthm}

\usepackage{algorithm,algpseudocode}
	\algnewcommand{\LeftComment}[1]{\Statex \(\triangleright\) #1}

\usepackage{subfloat}
\usepackage{enumerate}
\usepackage{multirow}
\usepackage{rotating}
\usepackage{subcaption}
\usepackage{needspace}

\input{mySymbol.sty}
\input{pennColors.sty}

\def \hlam {\hat{\lam}}

\newtheorem{theorem}{\hspace{0pt}\bf Theorem}

\newtheorem{definition}{\hspace{0pt}\bf Definition}

\title{Stability of Algebraic Neural Networks to Small Perturbations}

\name{Alejandro Parada-Mayorga and Alejandro Ribeiro}
\address{University of Pennsylvania}

\begin{document}
\maketitle
\begin{abstract}
Algebraic neural networks (AlgNNs) are composed of a cascade of layers each one associated to and algebraic signal model, and information is mapped between layers by means of a nonlinearity function. AlgNNs provide a generalization of neural network architectures where formal convolution operators are used, like for instance traditional neural networks (CNNs) and graph neural networks (GNNs). In this paper we study stability of AlgNNs on the framework of algebraic signal processing. We show how any architecture that uses a formal notion of convolution can be stable beyond particular choices of the shift operator, and this stability depends on the structure of subsets of the algebra involved in the model. We focus our attention on the case of algebras with a single generator. 
\end{abstract}
%
%
%




\section{Introduction}
\label{sec:intro}

Convolutional architectures play a central role on countless scenarios in machine learning and the numerical evidence that proves the advantages of using them is overwhelming~\cite{deeplearning_book,gamagnns}. Additonally, theoretical insights like the ones developed in~\cite{mallat_ginvscatt,bruna_iscn,bruna_groupinvrepconvnn} and~\cite{fern2019stability} have provided solid explanations about why such architectures work well. These analysis apparently different in nature, have been performed considering signals defined on different domains and with different notions of convolution but with remarkable similarities not only in the final results but also in how the derivations are performed, posing the question of whether there exists an explanation for this at a more structural level.

Several notions of stability have been considered widely in the literature, among which there are the first notion stability for traditional CNNs introduced in~\cite{mallat_ginvscatt} by Mallat, and similar notions adapted for invariant scattering networks in~\cite{bruna_iscn} and networks affected by group actions in~\cite{bruna_groupinvrepconvnn} on one side, while for irregular domains like GNNs notions of stability have been considered in~\cite{zou_stability,gamabruna_diffscattongraphs} and recently in~\cite{fern2019stability} where concrete stability calculations are performed.

In this work we discuss the notion of stability of neural structures exploiting algebraic signal models whose particular instantiations lead to the convolutional operators used in traditional CNNs and GNNs. We formulate the definition of stability using the language of the representation theory of algebras providing concrete calculations of stability when algebras with a single generator are considered, and we point out that those results for CNNs and GNNs can be obtained as particular cases. Our results highlight the universality of stability beyond particular instantiations of algebraic models whenever the algebra involved is the same. This explains the notorious similarities between the analysis and final results about stability for CNNs and GNNs. 

This paper is organized as follows. In Section 2 we discuss the basics about algebraic signal models and in Section 3 we introduce algebraic neural networks. In Section 4 we define the perturbations considered for the stability analysis and in Section 5 we discuss the stability results. Finally in section 6 we present some conclusions.



%


\section{Algebraic Signal Models} \label{sec_alg_filters}

An algebraic signal model (ASM) can be defined by the triple $(\ccalA,\ccalM,\rho),$ where $\ccalA$ is an associative algebra with unity, $\ccalM$ is a vector space with inner product, and $\rho:\ccalA\rightarrow\text{End}(\ccalM)$ is a homomorphism between $\ccalA$ and the set of endomorphisms in $\ccalM$~\cite{algSP0, algSP1, algSP2, algSP3}. In the context of representation theory of algebras, the pair $(\mathcal{M},\rho)$ is a representation of $\ccalA$. Notice that $\rho$ is a linear mapping that preserves the products in $\ccalA$.

In the ASM the elements in $\ccalM$ are the signals, the elements in $\ccalA$ are the filters and $\rho$ provides a physical realization associated to $\ccalM$  of the elements in $\ccalA$. Then, the filtered version of a signal $\bbx \in \mathcal{M}$ by an element in $a\in\ccalA$ is given by $ \bby = \rho(a) \bbx.$ As pointed out in~\cite{algSP0} the operation $ \rho(a) \bbx$ defines a general and formal notion of convolution between $\rho(a)$ and $\bbx$, which is the notion used in this paper to study stability in algebraic neural networks.

Any element in an algebra $\ccalA$ can be expressed as a polynomial function of the \textit{generators} of $\ccalA$. For the the discussion in this paper we focus on algebras that are generated by one element, i.e. all elements in $\ccalA$ can be expressed as $a=p_{\ccalA}(g)$ where $g$ is the generator of $\ccalA$ and $p_{\ccalA}$ is a polynomial function. As a consequence of the properties of $\rho$ as an homomorphism, we have that $\rho\left(p_{\ccalA}(g)\right)$ is a polynomial $p_{\ccalM}$ expressed in terms of $\rho(g)$, therefore $\rho\left(p_{\ccalA}(g)\right)=p_{\ccalM}\left(\rho(g)\right)$. Notice that the \textit{form} of $p_{\ccalA}$ and $p_{\ccalM}$ is the same, therefore for the sake of simplicity of drop the subindex to denote $\rho\left(p(g)\right)=p\left(\rho(g)\right)$. The operator $\rho(g)=\bbS$ is called shift operator.

Particular cases of the algebraic models lead to signal processing frameworks well known in the literature. For instance using the polynomial algebra, which has a single generator, it is possible to derive discrete time signal processing (DTSP), graph signal processing (GSP) and graphon signal processing (WSP) considering different representations~\cite{paradaalgnn,algSP0}, i.e. different choices of $\ccalM$ and $\rho$. DTSP uses a vector space (countably infinite dimensional) where the elements are sequences of square summable coefficients and the shift operator is the delay function, while in GSP the vector space is an $N-$dimensional vector space where $N$ is the number of nodes in the graph and the shift operator could be the adjacency matrix of the graph while in WSP the vector space is the set of functions with finite energy in the interval $[0,1]$ and the shift is associated to an integral transform whose kernel is given by a graphon~\cite{paradaalgnn,algSP0}.

%
%
\section{Algebraic Neural Networks}\label{sec_Algebraic_NNs}


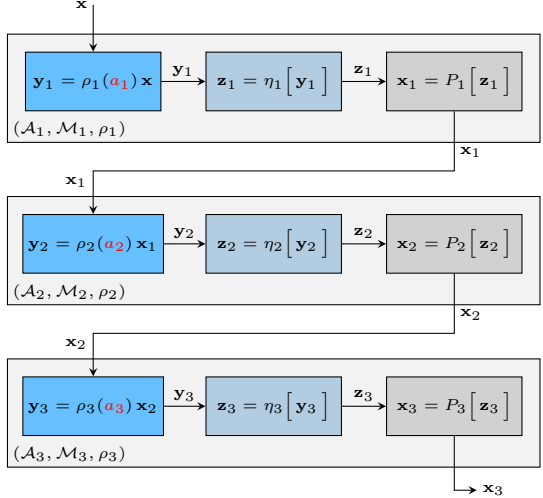
\begin{figure}
\centering

\definecolor{my_col1}{rgb}{0,0.5882,1}
\definecolor{my_col2}{rgb}{0.4980,0.6745,0.8000}
\definecolor{my_col3}{rgb}{0.6980,0.6980,0.6980}

\def \myfactor {0.6}
\def \unit  {\myfactor cm}

\tikzstyle{block} = [ rectangle,
                      minimum width = \unit,
                      minimum height = \unit,
                      fill = black,
                      draw = black,
                      text = black]

\tikzstyle{filter} = [block,
                      minimum width  = 3.0*\unit,
                      minimum height = 1.3*\unit,
                      fill=my_col1!60]

\tikzstyle{nonlinearity} = [ filter,
                             minimum width  = 3.0*\unit,
                             fill =my_col2!60]

\tikzstyle{pooling} = [ filter,
                             minimum width  = 3.0*\unit,
                             fill = my_col3!60]

\def \deltainput     {( 0.0,-1.7)}
\def \deltaoutput    {( 0.0,-1.2)}
\def \deltalayer     {3.6}
\def \deltaconnector {1.45}
\def \deltasigma     {( 4, 0.0)}

\def \one   {$\displaystyle{\mathbf{y}_{1}  = \rho_{1}(\red{a_1})\,\mathbf{x}}$}
\def \two   {$\displaystyle{\mathbf{y}_2  =  \rho_{2}(\red{a_2})\,\mathbf{x}_{1}}$}
\def \three {$\displaystyle{\mathbf{y}_3  = \rho_{3}(\red{a_3})\,\mathbf{x}_{2}}$}
\def \sigmaone   {$\displaystyle{\mathbf{z}_{1} = {\eta_{1}} \Big[\, \mathbf{y}_1 \, \Big]}$}
\def \sigmatwo   {$\displaystyle{\mathbf{z}_{2} = {\eta_{2}} \Big[\, \mathbf{y}_2 \, \Big]}$}
\def \sigmathree {$\displaystyle{\mathbf{z}_{3} = {\eta_{3}} \Big[\, \mathbf{y}_3 \, \Big]}$}
\def \proyone   {$\displaystyle{\mathbf{x}_{1} = {P_{1}} \Big[\, \mathbf{z}_1 \, \Big]}$}
\def \proytwo   {$\displaystyle{\mathbf{x}_{2} = {P_{2}} \Big[\, \mathbf{z}_2 \, \Big]}$}
\def \proythree   {$\displaystyle{\mathbf{x}_{3} = {P_{3}} \Big[\, \mathbf{z}_3 \, \Big]}$}

%
{\fontsize{6}{6}\selectfont\begin{tikzpicture}[scale = \myfactor]

  \pgfdeclarelayer{bg}     
  \pgfsetlayers{bg,main}   

  \node (input) [rectangle, minimum width = 0.1*\unit] {$\mathbf{x}$};
  \path (input.east)      ++ \deltainput node [filter]       (L1 Filter1) {\one};
  \path (L1 Filter1) ++ \deltasigma node [nonlinearity] (L1 F1)      {\sigmaone};
  \path (L1 F1) ++ \deltasigma node [pooling] (L1 F2)      {\proyone};
  \path[draw, -stealth] (L1 Filter1.east) -- node [above] {$\mathbf{y}_1$} (L1 F1.west);
  \path[draw, -stealth] (L1 F1.east) -- node [above] {$\mathbf{z}_1$} (L1 F2.west);

  \path (L1 Filter1) ++ (0,-\deltalayer) node [filter]       (L2 Filter1) {\two};
  \path (L2 Filter1) ++ \deltasigma      node [nonlinearity] (L2 F1)      {\sigmatwo};
  \path (L2 F1) ++ \deltasigma node [pooling] (L2 F2)      {\proytwo};
  \path[draw, -stealth] (L2 Filter1.east) --  node [above] {$\mathbf{y}_2$} (L2 F1.west);
  \path[draw, -stealth] (L2 F1.east) -- node [above] {$\mathbf{z}_2$} (L2 F2.west);  
  
  \path (L2 Filter1) ++ (0,-\deltalayer) node [filter]       (L3 Filter1) {\three};
  \path (L3 Filter1) ++ \deltasigma      node [nonlinearity] (L3 F1)      {\sigmathree};
  \path (L3 F1) ++ \deltasigma node [pooling] (L3 F2)      {\proythree};
  \path[draw, -stealth] (L3 Filter1.east) --  node [above] {$\mathbf{y}_3$} (L3 F1.west);
  \path[draw, -stealth] (L3 F1.east) -- node [above] {$\mathbf{z}_3$} (L3 F2.west); 

  \path[draw, -stealth] (input.east) -- (L1 Filter1.north);
  \path (L1 F2.south) ++ (0,-\deltaconnector) node [] (aux1) {};
  \path[draw, -stealth] (L1 F2.south) -- node [below right] {$\mathbf{x}_1$} (aux1.north) 
                                      --                         (aux1.north -| L2 Filter1.north) 
                                      -- node [above left]  {$\mathbf{x}_1$} (L2 Filter1.north);
  \path (L2 F2.south) ++ (0,-\deltaconnector) node [] (aux1) {};
  \path[draw, -stealth] (L2 F2.south) -- node [below right] {$\mathbf{x}_2$} (aux1.north) 
                                      --                         (aux1.north -| L2 Filter1.north) 
                                      -- node [above left]  {$\mathbf{x}_2$} (L3 Filter1.north);
  \path[draw, -stealth] (L3 F2.south) -- ++ \deltaoutput -- ++ (0.5, 0) 
                        node [right]{$\mathbf{x}_3$};

  \begin{pgfonlayer}{bg} 
      \path (L1 Filter1.west |- L1 F1.south) ++ (-0.4,-0.7)
           node [filter, anchor = south west,
                 fill = black!5, 
                 minimum width  = 11.8*\unit,
                 minimum height = 2.4*\unit,] 
        (layer)
        {}; 
       \path (layer.south west) ++ (0.0,0.0) node [above right] {$(\ccalA_1, \ccalM_1, \rho_1)$};
      \path (L1 Filter1.west |- L2 F1.south) ++ (-0.4,-0.7)
           node [filter, anchor = south west,
                 fill = black!5, 
                 minimum width  = 11.8*\unit,
                 minimum height = 2.4*\unit,] 
        (layer)
        {}; 
       \path (layer.south west) ++ (0.0,0.0) node [above right] {$(\ccalA_2, \ccalM_2, \rho_2)$};
      \path (L1 Filter1.west |- L3 F1.south)  ++ (-0.4,-0.7)
           node [filter, anchor = south west,
                 fill = black!5, 
                 minimum width  = 11.8*\unit,
                 minimum height = 2.4*\unit,] 
        (layer)
        {}; 
       \path (layer.south west) ++ (0.0,0.0) node [above right] {$(\ccalA_3, \ccalM_3, \rho_3)$};  \end{pgfonlayer}

\end{tikzpicture}} 
  \caption{Example of an Algebraic Neural Network $\Xi=\{(\ccalA_{\ell},\ccalM_{\ell},\rho_{\ell})\}_{\ell=1}^{3}$ with three layers. The input signal $\mathbf{x}$ is mapped by $\Xi$  into $\mathbf{x}_{3}$.}
  \label{fig_6}
\end{figure}

An algebraic neural network (AlgNN) consists of a stacked layered structure, where each layer is associated to an algebraic signal model. Then, if we have an AlgNN with $L$ layers, there is an algebraic model $(\ccalA_{\ell},\ccalM_{\ell},\rho_{\ell})$ associated to the $\ell$th layer. The information is mapped from the layer $\ell$ to the layer $\ell+1$ by a map $\sigma_{\ell}=P_{\ell}\circ\eta_{\ell}$ that is the composition of a point-wise nonlinearity and a pooling operator. For our discussion we consider $\sigma_{\ell}$ as a Lipschitz map with zero as fixed point.  In Fig.~\ref{fig_6}  a pictorial representation of an AlgNN is presented. Notice that the relationship between the output signal of the layer $\ell-1$ and the output signal of the layer $\ell$ is then give by 
\begin{equation}\label{eq:xl}
\bbx_{\ell}=\sigma_{\ell}\left(\rho_{\ell}(a_{\ell})\bbx_{\ell-1}\right)=\Phi (\mathbf{x}_{\ell-1},\mathcal{P}_{\ell-1},\mathcal{S}_{\ell-1}),
\end{equation}
with $a_{\ell}\in\ccalA_{\ell}$. We use the symbol $\Phi (\mathbf{x}_{\ell-1},\mathcal{P}_{\ell-1},\mathcal{S}_{\ell-1})$ to emphasize that the filters used in the processing of the signal $\bbx_{\ell-1}$ belong to the subset $\ccalP_{\ell}\subset\ccalA_{\ell}$. This will become relevant once we state stability  results. The map associated to the whole AlgNN is represented by $\Phi\left(\mathbf{x},\{ \mathcal{P}_{\ell} \}_{1}^{L},\{ \mathcal{S}_{\ell}\}_{1}^{L}\right)$ acting on an input signal $\bbx$. We point out that the filtering in each layer of the AlgNN can be performed by means of several filter banks. This allows the generation of several features, case in which the output signal associated to the feature $f$ in the layer $\ell$ is given by $\bbx_{\ell}^{f}$.

\section{Deformations in Algebraic Signal Models}\label{sec:perturbandstability}

In previous sections we mentioned the role of $\rho$ in the algebraic signal model $(\ccalA,\ccalM,\rho)$ as a physical realization of the algebraic filters in $\ccalA$, therefore it is natural in this context to associate the perturbation of the filters to $\rho$. One can think that $\rho$ will provide an ideal realization of the filters in $\ccalA$ but that in practice one gets a realization $\tilde{\rho}$ that is an approximate version of $\rho$. We consider deformations or perturbations of ASM analyzing changes in $\rho$. More specifically, we say that the triple $(\ccalA,\ccalM,\tilde{\rho})$ is a perturbed version of $(\ccalA,\ccalM,\rho)$ where $\tilde{\rho}$ is defined according to
\begin{equation}\label{eqn_def_perturbation_model_10}
   \tdrho(p(g)) = p\big(\tdrho(g)\big) 
             = p\big(\tilde\bbS\big)
\end{equation}
 where $\tilde\bbS$  is related to the shift operator $\bbS=\rho(g)$ by means of $\tbS = \bbS + \bbT(\bbS)$, and the term $\bbT(\bbS)$ is called the perturbation associated to $(\ccalA,\ccalM,\tilde{\rho})$. It is important to point out that $\tilde{\rho}$ is not necessarily a homomorphism. As pointed out in~\cite{paradaalgnn} perturbations in the domain of the signals can be reframed in terms of this perturbation model. 
 
 For our discussion we focus on deformations of the type
 \begin{equation}\label{eqn_perturbation_model_absolute_plus_relative}
   \bbT(\bbS)=\epsilon\bbI+ \bbT_{1}\bbS,
\end{equation}
 which consists of an additive or absolute perturbation $\epsilon\bbI$ where $\epsilon>0$ is a small scalar and $\bbI$ is the identity matrix  and a relative perturbation $\bbT_{1}\bbS$. The operator $\bbT_{1}$ is a compact normal operator with operator norm $\Vert\bbT_{1}\Vert \ll1$. The measure of the non commutativity between the operator $\bbT_{1}$ and $\bbS$ is modeled as $\mathbf{S}\mathbf{T}_{1}=\mathbf{T}_{c1}\mathbf{S}+\mathbf{S}\mathbf{P}_{1},
$
 %
 %
%
%
and by the factor $\delta$ given by
 \begin{equation}\label{eqn_commutation_factor}
   \delta = \max \frac{\| \bbP_{1}\|}{\| \bbT_1\|},
\end{equation}
 where $\mathbf{T}_{c1}=\sum_{i}\mu_{i}\mathbf{u}_{i}\langle\mathbf{u}_{i},\cdot\rangle$, ($\mu_{i}$, $\mathbf{u}_{i}$) is the $i$th eigenpair $\mathbf{T}_{1}$, $\mathbf{u}_{i}$ are the eigenvectors of $\mathbf{S}$, and $\langle,\rangle$ is the inner product. As proven in~\cite{paradaalgnn} the value of $\delta$ is upper bounded by the weighted difference between the eigenvectors of $\bbx$ and $\bbT_{1}$.


\section{Stability Results}\label{sec:stabilitytheorems}

Now we introduce the notion of stability for our discussion. We emphasize the notation presented in previous section, where $(\ccalA,\ccalM,\tilde{\rho})$ denotes the perturbed version of $(\ccalA,\ccalM,\rho)$, $\rho(g)=\bbS$, $\tilde{\rho}(g)=\tilde{\bbS}$ being $g$ the generator of the  algebra. 


\begin{definition}[]\label{def:stabilityoperators1} 
Given  $p(\mathbf{S})$ and $p(\tilde{\mathbf{S}})$ defined on the algebraic signal model $(\ccalA,\ccalM,\rho)$ and $(\ccalA,\ccalM,\tdrho)$  we say the operator $p(\mathbf{S})$ is Lipschitz stable if there exist constants $C_{0},C_{1}>0$ such that 
\begin{multline}\label{eq:stabilityoperators1}
\left\Vert p(\mathbf{S})\mathbf{x}  - p(\tilde{\mathbf{S}})\mathbf{x}\right\Vert
\leq
\\
\left[
C_{0}\sup_{\bbS\in\ccalS}\Vert\bbT(\bbS)\Vert+C_{1}\sup_{\bbS\in\ccalS}\big\|D_{\bbT}(\bbS)\big\|
+\mathcal{O}\left(\Vert\mathbf{T}(\mathbf{S})\Vert^{2}\right)
\right] \big\| \bbx \big\|,
\end{multline}
for all $\bbx\in\ccalM$. In \eqref{eq:stabilityoperators1} $D_{\bbT}(\bbS)$ is the Fr\'echet derivative of the perturbation operator $\bbT$. 
\end{definition}

We point out that the right hand side in eqn.~(\ref{eq:stabilityoperators1}) provides a measure of the size of the perturbation $\bbT(\bbS)$. Then, the definition~\ref{def:stabilityoperators1} says that the operators representing a given algebraic filter are stable to the action of a perturbation if the change in the operator is proportional to the size of the perturbation. It is worth pointing out that this is the same principle considered in the concept of stability presented in~\cite{mallat_ginvscatt}. Additionally, as a particular case of this definition we obtain also the notion of stability used for GNNs in~\cite{fern2019stability}.

\subsection{Characterizing subsets of the algebra}

While determining if subsets of filters in the algebra are stable or not, a characterization of such subsets is necessary. For algebras with a single generator this can be done using functions of an independent variable. It is worth pointing out that when considering algebras like the polynomial algebra of a single variable, generated by one element, it is not necessary to assume that the independent variable of the algebra takes values in a particular field. However, this type of assumption can be useful to define subsets of the algebra. For our discussion we assume that the independent variable of the algebra takes values in $\mathbb{C}$. 

In the the following definition, we describe the properties of the subsets of filters involved in the stability calculations.


\begin{definition}
Let $p:\mathbb{C}\rightarrow\mathbb{C}$ be a one variable function. Then,  we say that $p$ is Lipschitz if there exists $L_{0}>0$ such that
\begin{equation}
\vert p(\lambda)-p(\mu)\vert\leq L_{0}\vert\lambda-\mu\vert
\end{equation}
for all $\lambda, \mu\in\mathbb{C}$. Additionally, we say that $p(\lambda)$ is Lipschitz integral if there exists $L_{1}>0$ such that
\begin{equation}
\left\vert 
\lambda\frac{dp(\lambda)}{d\lambda}
\right\vert
\leq L_{1}
~\forall~\lambda.
\end{equation}
\end{definition}

We denote by $\mathcal{A}_{L_{0}}$ the subset of $\ccalA$ whose elements are Lipschitz functions with constant $L_{0}$ and by $\mathcal{A}_{L_{1}}$ the subset of elements in $\mathcal{A}$ that are Lipschitz integral with constant $L_{1}$.

\subsection{Stability of Algebraic Filters}

Before stating concrete stability calculations for the particular type of deformation we are considering, we introduce a theorem, that provides essential insights for our discussion.


\begin{theorem}[]\label{theorem:HvsFrechet}
Let $\mathcal{A}$ be an algebra generated by $g$ and let $(\mathcal{M},\rho)$ be a representation of $\mathcal{A}$ with $\rho(g)=\bbS\in\text{End}(\mathcal{M})$. Let $\tilde{\rho}(g)=\tilde{\bbS}\in\text{End}(\mathcal{M})$ where the pair
$(\mathcal{M},\tilde{\rho})$ is a perturbed version of $(\mathcal{M},\rho)$ and $\tilde{\bbS}$ is related to $\bbS$ by the perturbation model in eqn.~(\ref{eqn_perturbation_model_absolute_plus_relative}). Then, for any $p\in\mathcal{A}$ we have
\begin{equation}
\left\Vert p(\bbS)\mathbf{x}-p(\tilde{\bbS})\mathbf{x}\right\Vert
\leq 
\Vert\mathbf{x}\Vert
\left(
\left\Vert D_{p}(\mathbf{S})\left\lbrace\mathbf{T}(\mathbf{S})\right\rbrace\right\Vert + \mathcal{O}\left(\Vert\mathbf{T}(\mathbf{S})\Vert^{2}
\right)\right)
\label{eq:HSoptbound}
\end{equation}
where $D_{p}(\mathbf{S})$ is the Fr\'echet derivative of $p$ on $\mathbf{S}$.
\end{theorem}
\begin{proof}
See~\cite{paradaalgnn}
\end{proof}
Theorem~\ref{theorem:HvsFrechet} show how the filters act directly on the deformation. In particular, we can see from the therm $\left\Vert D_{p}(\mathbf{S})\left\lbrace\mathbf{T}(\mathbf{S})\right\rbrace\right\Vert $ that the Fr\'echet derivative of the physical representation of the filter acts directly on $\bbT(\bbS)$, and this is independent from the properties of $\bbT(\bbS)$, i.e.  the perturbation is not restricted to be the one we are considering for our discussion. 

Now we present a theorem where an upper bound of the right hand side of eqn.~(\ref{eq:HSoptbound}) is determined in terms of the properties of $\bbT(\bbS)$ for a specific type of filters.


\begin{theorem}\label{theorem:uppboundDH}
Let $\mathcal{A}$ be an algebra with one generator element $g$ and let $(\mathcal{M},\rho)$ be a finite or countable infinite dimensional representation of $\mathcal{A}$. Let  $(\mathcal{M},\tilde{\rho})$ be a perturbed version of $(\mathcal{M},\rho)$ associated to the perturbation model in eqn.~(\ref{eqn_perturbation_model_absolute_plus_relative}).  If $p_{\mathcal{A}}\in\mathcal{A}_{L_{0}}\cap\mathcal{A}_{L_{1}}$, then
\begin{equation}
\left\Vert D_{p}\bbT(\bbS)\right\Vert
\leq(1+\delta)
\left(L_{0}\sup_{\bbS}\Vert\bbT(\bbS)\Vert +L_{1}\sup_{\bbS}\Vert D_{\bbT}(\bbS)\Vert\right)
\label{eq:DHTS}
\end{equation}
 \end{theorem}
 \begin{proof}See~\cite{paradaalgnn}\end{proof}

From theorem~\ref{theorem:uppboundDH} and taking into account Theorem~\ref{theorem:HvsFrechet}, we can conclude that AlgNNs are stable to the perturbations considered, indeed the norm of the Fr\'echet derivative of the filter acting on the perturbation is bounded by the size of the perturbation. Notice also that the non commutativity between the operators $\bbS$ and $\bbT_{1}$ does not change the functional form of the upper size in eqn.~(\ref{eq:DHTS}), however it does increase the size of the constants.

\subsection{Stability of Algebraic Neural Networks}

We have shown that algebraic filters defined in the algebraic context are stable to the deformations considered in the model~(\ref{eqn_perturbation_model_absolute_plus_relative}). In this section we discuss how this results are inherited by AlgNN. This is associated directly to the fact that the functions $\sigma_{\ell}$ mapping information  between layers are Lipschitz functions. We start pointing out the effect of the functions $\sigma_{\ell}$ mapping information between the layers of the AlgNN.  Before doing so we clarify that the AlgNN $\Xi=\left\lbrace (\mathcal{A}_{\ell},\mathcal{M}_{\ell},\rho_{\ell})\right\rbrace_{\ell=1}^{L}$ is perturbed into $\tilde{\Xi}=\left\lbrace (\mathcal{A}_{\ell},\mathcal{M}_{\ell},\tilde{\rho}_{\ell})\right\rbrace_{\ell=1}^{L}$ if $(\mathcal{A}_{\ell},\mathcal{M}_{\ell},\tilde{\rho}_{\ell})$ is a perturbed version of $(\mathcal{A}_{\ell},\mathcal{M}_{\ell},\rho_{\ell})$.


\begin{theorem}\label{theorem:stabilityAlgNN0}
Let $\Xi=\left\lbrace (\mathcal{A}_{\ell},\mathcal{M}_{\ell},\rho_{\ell})\right\rbrace_{\ell=1}^{L}$ be an algebraic neural network  with $L$ layers, one feature per layer and algebras $\mathcal{A}_{\ell}$ with a single generator.  Let  $\tilde{\Xi}=\left\lbrace (\mathcal{A}_{\ell},\mathcal{M}_{\ell},\tilde{\rho}_{\ell})\right\rbrace_{\ell=1}^{L}$ be the perturbed version of $\Xi$ by means of the perturbation model in eqn.~(\ref{eqn_perturbation_model_absolute_plus_relative}). Then, if  $\Phi\left(\mathbf{x}_{\ell-1}, \mathcal{P}_{\ell},\mathcal{S}_{\ell}\right)$ and 
$\Phi\left(\mathbf{x}_{\ell-1},\mathcal{P}_{\ell},\tilde{\mathcal{S}}_{\ell}\right)$ represent the mapping operators associated to $\Xi$ and $\tilde{\Xi}$ in the layer $\ell$ respectively, we have
\begin{multline}
\left\Vert
\Phi\left(\mathbf{x}_{\ell-1},\mathcal{P}_{\ell},\mathcal{S}_{\ell}\right)-
\Phi\left(\mathbf{x}_{\ell-1},\mathcal{P}_{\ell},\tilde{\mathcal{S}}_{\ell}\right)
\right\Vert
\leq
\\
C_{\ell}(1+\delta_{\ell})\left(L_{0}^{(\ell)} \sup_{\bbS_{\ell}}\Vert\bbT^{(\ell)}(\bbS_{\ell})\Vert 
+L_{1}^{(\ell)}\sup_{\bbS_{\ell}}\Vert D_{\mathbf{T^{(\ell)}}}(\bbS_{\ell})\Vert\right)\Vert\mathbf{x}_{\ell-1}\Vert
\label{eq:theoremstabilityAlgNN0}
\end{multline}
where $C_{\ell}$ is the Lipschitz constant of $\sigma_{\ell}$, and $\mathcal{P}_{\ell}=\mathcal{A}_{L_{0}}\cap\mathcal{A}_{L_{1}}$ represents the domain of $\rho_{\ell}$. 
The index $(\ell)$ makes reference to quantities and constants associated to the layer $\ell$.
\end{theorem}
\begin{proof}See~\cite{paradaalgnn}\end{proof}

From theorem~\ref{theorem:stabilityAlgNN0} we see that the functions $\sigma_{\ell}$ do not affect the stability properties of the algebraic filters although a modification of the constants is produced. This highlights the goodness of $\sigma_{\ell}$ as they do not affect the stability and enrich the way the network handles information with different spectral content allowing an arbitrary degree of selectivity. 

We introduce our final result of stability for a general AlgNN with $L$ layers and one generator.


\begin{theorem}\label{theorem:stabilityAlgNN1}
Let $\Xi=\left\lbrace (\mathcal{A}_{\ell},\mathcal{M}_{\ell},\rho_{\ell})\right\rbrace_{\ell=1}^{L}$ be an algebraic neural network  with $L$ layers, one feature per layer and algebras $\mathcal{A}_{\ell}$ with a single generator.  Let  $\tilde{\Xi}=\left\lbrace (\mathcal{A}_{\ell},\mathcal{M}_{\ell},\tilde{\rho}_{\ell})\right\rbrace_{\ell=1}^{L}$ be the perturbed version of $\Xi$ by means of the perturbation model in eqn.~(\ref{eqn_perturbation_model_absolute_plus_relative}). Then, if  $\Phi\left(\mathbf{x},\{ \mathcal{P}_{\ell} \}_{1}^{L},\{ \mathcal{S}_{\ell}\}_{1}^{L}\right)$ and 
$\Phi\left(\mathbf{x},\{ \mathcal{P}_{\ell} \}_{1}^{L},\{ \tilde{\mathcal{S}}_{\ell}\}_{1}^{L}\right)$ represent the mapping operators associated to $\Xi$ and $\tilde{\Xi}$ respectively, we have
\begin{multline}
\left\Vert
\Phi\left(\mathbf{x},\{ \mathcal{P}_{\ell} \}_{1}^{L},\{ \mathcal{S}_{\ell}\}_{1}^{L}\right)-
\Phi\left(\mathbf{x},\{ \mathcal{P}_{\ell} \}_{1}^{L},\{ \tilde{\mathcal{S}}_{\ell}\}_{1}^{L}\right)
\right\Vert
\\
\leq
\sum_{\ell=1}^{L}\boldsymbol{\Delta}_{\ell}\left(\prod_{r=\ell}^{L}C_{r}\right)\left(\prod_{r=\ell+1}^{L}B_{r}\right)
\left(\prod_{r=1}^{\ell-1}C_{r}B_{r}\right)\left\Vert\mathbf{x}\right\Vert
\label{eq:theoremstabilityAlgNN1}
\end{multline}
where $C_{\ell}$ is the Lipschitz constant of $\sigma_{\ell}$, $\Vert\rho_{\ell}(\xi)\Vert\leq B_{\ell}~\forall~\xi\in\mathcal{P}_{\ell}$, and $\mathcal{P}_{\ell}=\mathcal{A}_{L_{0}}\cap\mathcal{A}_{L_{1}}$ is the domain of $\rho_{\ell}$. The functions $\boldsymbol{\Delta}_{\ell}$ are given by
\begin{equation}\label{eq:varepsilonl}
\boldsymbol{\Delta}_{\ell}=(1+\delta_{\ell})\left(L_{0}^{(\ell)} \sup_{\bbS_{\ell}}\Vert\bbT^{(\ell)}(\bbS_{\ell})\Vert 
+L_{1}^{(\ell)}\sup_{\bbS_{\ell}}\Vert D_{\mathbf{T^{(\ell)}}}(\bbS_{\ell})\Vert\right)
\end{equation}
%
%
%
%
with the index $(\ell)$ indicating quantities and constants associated to the layer $\ell$.
\end{theorem}
\begin{proof}See Section~\cite{paradaalgnn}\end{proof}

This final result provides a guarantee of stability for AlgNNs at an algebraic level. It is worth remarking that:
\begin{itemize}
\item The representations used in each layer can be substantially different from one another, i.e. the pairs $(\ccalM_{\ell},\rho_{\ell})$ can change substantially from layer $\ell$ to layer $\ell+1$, and still the stability results hold. This highlights the rich variety of options for the representation and processing of information with different types of operators without affecting the stability.
\item A particular instantiation of the algebraic model considering the polynomial algebra and a vector space whose dimension equals the number of nodes in a graph, leads to results of stability for GNNs obtained in~\cite{fern2019stability}.
\item The fact that AlgNNs inherit stability from the stability of the algebraic operators resides on the properties of the maps $\sigma_{\ell}$ and the composition of several layers does not affect the functional expressions defining the stability but instead they increase the size of the constants.
\end{itemize}

We remark that this results are generalized for algebras with multiple generators, the reader interested in this extension can check the details of the proofs in~\cite{paradaalgnn} where the analysis is performed considering algebras with multiple generators.


%
%
%
%
%
%
%

\section{Conclusions}\label{sec:conclusions}

We discussed algebraic neural networks (AlgNNs) considering algebras with a single generator in order to generalize the notion of stability when considering general convolution operators. Particular choices of the algebras and their representations in a AlgNN  lead to traditional CNNs and GNNs, therefore the results obtained in this paper provide a unifying framework for the stability analysis. Indeed, we can say that stability transcend particular choices of spaces and operators whenever a formal representation is behind the structure of each layer. Additionally, we show that the physical realization of the algebraic filters as elements of $\text{End}(\ccalM)$ acts directly on the perturbation by means of its Fr\'echet derivative, and it is precisely the action of this operator on the perturbation what determines the structure of the subsets of the algebra that will guarantee stability.

%
%

%
%
%

%
%

\clearpage
\bibliographystyle{IEEEbib}
\bibliography{bibliography}
\clearpage

\end{document}